\documentclass{article}

\usepackage{arxiv}

\usepackage[utf8]{inputenc} 
\usepackage[T1]{fontenc}    
\usepackage{hyperref}       
\usepackage{url}            
\usepackage{booktabs}       
\usepackage{amsfonts}       
\usepackage{nicefrac}       
\usepackage{microtype}      
\usepackage{lipsum}

\usepackage{graphicx}
\usepackage{subfigure}

\usepackage{amsfonts}
\usepackage{amsmath}
\usepackage{amsthm}
\usepackage{amssymb}
\usepackage{enumerate}
\usepackage{multirow}

\usepackage{algorithm}
\usepackage{algpseudocode}

\usepackage[authoryear]{natbib}

\newtheorem{thm}{Theorem}[section]

\newtheorem{lem}[thm]{Lemma}
\newtheorem{prop}[thm]{Proposition}

\usepackage{hyperref}

\title{Contextual Multi-armed Bandit Algorithm for Semiparametric Reward Model}

\author{
  Gi-Soo Kim \\
  Department of Statistics\\
  Seoul National University\\
  Seoul, Korea \\
  \texttt{gisoo1989@snu.ac.kr} \\
   \And
  Myunghee Cho Paik* \\
  Department of Statistics\\
  Seoul National University\\
  Seoul, Korea \\
  \texttt{myungheechopaik@snu.ac.kr} \\
}

\begin{document}
\maketitle

\begin{abstract}
Contextual multi-armed bandit (MAB) algorithms have been shown promising for maximizing cumulative rewards in sequential decision tasks such as news article recommendation systems, web page ad placement algorithms, and mobile health. 
However, most of the proposed contextual MAB algorithms assume linear relationships between the reward and the context of the action. This paper proposes a new contextual MAB algorithm for a relaxed, semiparametric reward model that supports nonstationarity. The proposed method is less restrictive, easier to implement and faster than two alternative algorithms that consider the same model, while achieving a tight regret upper bound. We prove that the high-probability upper bound of the regret incurred by the proposed algorithm has the same order as the Thompson sampling algorithm for linear reward models. 
The proposed and existing algorithms are evaluated via simulation and also applied to Yahoo! news article recommendation log data.
\end{abstract}

\keywords{Contextual multi-armed bandit algorithm \and Thompson sampling \and Semiparametric model}

\section{Introduction}\label{s1}

The multi-armed bandit (MAB) problem \citep{Robbins} formulates the sequential decision problem in which a learner must choose an action among several actions given by the environment at each step so as to maximize the cumulative rewards. The actions are often described as the arms of a bandit slot machine with multiple arms. By choosing an action or pulling an arm, the learner receives possibly different rewards. By repeating the process of pulling arms and receiving rewards, the learner accumulates information about the reward compensation mechanism,  learns from it, and chooses the arm  close to optimal as time elapses. 
Application areas include the mobile healthcare system \citep{Tewari and Murphy}, web page ad placement algorithms \citep{Langford}, news article placement algorithms \citep{Li10}, revenue management \citep{Ferreira}, marketing \citep{Schwartz}, and recommendation systems \citep{Kawale}. 

For example, the Yahoo! web system uses a news article recommendation algorithm to select one article among a large pool of available articles and displays it on the Featured tab of the web page every time a user visits. The user clicks the article if he or she is interested in the contents. The goal of the algorithm is to maximize the cumulative number of user clicks. After each visit, the algorithm reinforces its article selection strategy based on the past user click feedback. In this setting, available articles correspond to different actions and the user click corresponds to a reward. The challenging part of the MAB problem is that the reward of the action that the learner has not previously chosen is forever unknown, i.e., whether the user would have clicked or not remains missing for the article that is not chosen. Therefore, the learner should balance  between ``exploitation", selecting the best action based on information accumulated so far, and ``exploration", choosing an action that will assist in future choices, although it does not seem to be the best option at the moment. 

The MAB problem was first theoretically analyzed by \citet{Lai85}. The algorithms widely used in mobile healthcare systems or  news article placement algorithms are of a more extended form, called contextual MAB algorithms. A contextual MAB algorithm enables at each selection step the use of side information, called context, about each action given in the form of finite-dimensional covariates. For example, in the news article recommendation, information on the visiting user as well as the articles are given in the form of a context vector. In 2010, the Yahoo! team \citep{Li10} proposed a contextual MAB algorithm that achieved a 12.5\% click lift compared to a context-free MAB algorithm. Nonetheless, the method of \citet{Li10} and other existing algorithms rely on rather strong assumptions on the distribution of rewards. In particular, most of the existing algorithms assume that the expectation of the reward of a particular action has a time-invariant linear relationship with the context vector. This assumption can be restrictive in real world settings where the rewards typically {vary with time}.  

In this paper, we propose a novel contextual MAB algorithm which works well under a relaxed assumption on the distribution of rewards. The  relaxed nature of the assumption involves nonstationarity of the reward via an additive intercept term to the original time-invariant linear term. This intercept term changes with time but does not depend on the action. We propose a consistent estimation of the regression parameter in the linear term by centering the context vectors with weights. Using new martingale inequalities, we prove  that the high probability upper bound of the total regret incurred by the proposed algorithm has the same order as the regret bound achieved by the Thompson sampling algorithm which is developed under a more restrictive linear assumption. 

{\citet{Greenewald} and \citet{Krishnamurthy} suggested algorithms under the same nonstationary assumption we considered. The performance of the method by \citet{Greenewald} is guaranteed under restrictive conditions on the action choice probabilities. The method by \citet{Krishnamurthy} takes an action-elimination approach which is computationally heavy as it requires $O(N^2)$ computations at each iteration where $N$ denotes the number of arms. Moreover in \citet{Krishnamurthy}, the action selection distribution is not given explicitly when $N>2$. Our method improves on these previous results in that it does not restrict action choice probabilities, requires $O(N)$ computations, and explicitly provides the action selection distribution for every $N$. Furthermore, the proposed estimator for the regression parameter achieves the same convergence rate as the estimator for linear reward models. }
 
{ As a summary, our main contributions are:
\begin{itemize}
\item We propose a new MAB algorithm for the nonstationary semiparametric reward model. The proposed method is less restrictive, easier to implement and computationally faster than previous works.
\item We prove that the high-probability upper bound of the regret for the proposed method is of the same order as the Thompson Sampling algorithm for linear reward models.
\item {We propose a new estimator for the regression parameter  without requiring an extra tuning parameter and prove that it converges to the true parameter faster than existing estimators.}
\item Simulation studies show that in most cases, the cumulative reward of the proposed method increases faster than existing methods which assume the same nonstationary reward model. Application to Yahoo!  news  article  recommendation  log  data shows that the proposed method increases the user click rate compared to the algorithms that assume a stationary reward model.
\end{itemize}
}

 

\section{Preliminaries}\label{s2}
{In this section, we describe the problem settings and notations. As a preliminary, we also present a review of contextual bandit methods for the comparison purpose with the proposed method given in Section \ref{s4}.}

In the MAB setting, the learner is repeatedly faced with $N$ alternative actions where at time $t$, the $i$-th arm ($i=1,\cdots,N$) yields a random reward $r_i(t)$ with unknown mean $\theta_i(t)$. In the contextual MAB problem, we assume that there is a finite-dimensional context vector $b_i(t)\in \mathbb{R}^d$ associated with each arm $i$ at time $t$ and that the mean of $r_i(t)$ depends on $b_i(t)$, i.e., $\theta_i(t)=\theta_t(b_i(t)),$
where $\theta_t(\cdot)$ is an arbitrary function. 
Among the $N$ arms, the learner pulls one arm $a(t)$, and observes reward $r_{a(t)}(t)$. The optimal arm at time $t$ is $a^*(t):=\underset{1\leq i \leq N}{\mathrm{argmax}}\{\theta_t(b_i(t))\}$. Let $regret(t)$ be the difference between the expected reward of the optimal arm and the expected reward of the arm chosen by the learner at time $t$, i.e., 
\begin{align*}regret(t)&=\mathbb{E}\big(r_{a^*(t)}(t)-r_{a(t)}(t)~\big|~\{b_i(t)\}_{i=1}^N, a(t)~\big)\\&=\theta_t(b_{a^*(t)}(t))-\theta_t(b_{a(t)}(t)).\end{align*}
Then, the goal of the learner is to minimize the sum of regrets over $T$ steps,
$R(T):=\sum_{t=1}^Tregret(t).$

Linear contextual MAB problems specifically assume that $\theta_t(b_i(t))$ is linear in $b_i(t)$, 
\begin{align}\theta_t(b_i(t))=b_i(t)^T\mu,~~~i=1,\cdots,N,\label{linear}\end{align}
where $\mu\in\mathbb{R}^d$ is unknown. 
For the linear contextual MAB problem, \citet{Dani} proved a lower bound of order $\Omega(d\sqrt{T})$ for the regret $R(T)$ when $N$ is allowed to be infinite. When $N$ is finite and $d^2\leq T$ , \citet{Chu11} showed a lower bound of $\Omega(\sqrt{dT})$.


\citet{Auer}, \citet{Li10} and \citet{Chu11} proposed an upper confidence bound (UCB) algorithm for the linear contextual MAB problem. The algorithm selects the arm which has the highest UCB of the reward. Since the UCB reflects the current estimate of the reward as well as its uncertainty, the algorithm naturally balances between exploitation and exploration. The success of the UCB algorithm hinges on a valid upper confidence bound $U_i(t)$ of the $i$-th arm's reward, $b_i(t)^T\mu$. \citet{Li10} and \citet{Chu11} proposed 
$$U_i(t)=b_i(t)^T\hat{\mu}(t)+\alpha s_{t,i},$$
where $\hat{\mu}(t)$ is the regression estimator of $\mu$ at time $t$,
\begin{align}\hat{\mu}(t)=B(t)^{-1}\sum_{\tau=1}^{t-1}b_{a(\tau)}(\tau)r_{a(\tau)}(\tau),\label{muhat}\end{align}
$B(t)=I_d+\sum_{\tau=1}^{t-1}b_{a(\tau)}(\tau)b_{a(\tau)}(\tau)^T$ and $s_{t,i}=\sqrt{b_i(t)^TB(t)^{-1}b_i(t)}.$

Under additional assumptions that the error $\eta_i(t):=r_i(t)-\mathbb{E}(r_i(t)|b_i(t))=r_i(t)-b_i(t)^T\mu$ is $R$-sub-Gaussian for some $R>0$ and that the $L_2$-norms of $b_i(t)$ and $\mu$ are bounded by $1$, \citet{Abbasi-Yadkori} proved that if we set $\alpha=R\sqrt{3d\mathrm{log}({T}/{\delta})}+1$,  $U_i(t)$ is a $(1-\delta)$-probability upper bound of $b_i(t)^T\mu$ for $\forall \delta \in (0,1)$, for all $i=1,\cdots,N$ and $t=1,\cdots,T$. Since the errors $\eta_{a(\tau)}(\tau)$'s of the observed rewards are intercorrelated, \citet{Abbasi-Yadkori} used a concentration inequality for  vector-valued martingales to derive a tight $\alpha$.
Additionally, \citet{Abbasi-Yadkori} proved that with probability at least $1-\delta$, the UCB algorithm achieves, 
\begin{align}R(T)&\leq O\big(d\sqrt{T\mathrm{log}({T}/{\delta})\mathrm{log}(1+{T}/{d})}\big).\label{ucbregret2}\end{align}
The bound (\ref{ucbregret2}) matches the lower bound $\Omega(d\sqrt{T})$ for infinite $N$ by a factor of $\mathrm{log}(T)$. When $N$ is finite, (\ref{ucbregret2}) is slightly higher than the lower bound $\Omega(\sqrt{dT})$ by a factor of $\sqrt{d}\mathrm{log}(T).$
 

As a randomized version of the UCB algorithm, Thompson sampling \citep{Thompson} has been widely used as a simple heuristic based on Bayesian ideas. 
\citet{Agrawal} was the first to propose and analyze the Thompson sampling (TS) algorithm for linear contextual MABs. 

The heuristic of the algorithm is to randomly pull the arm according to the posterior probability that it is the optimal arm. This can be done by sampling $\tilde{\mu}(t)$ from the posterior distribution of $\mu$ at time $t$, and pulling the arm $a(t)=\underset{1\leq i\leq N}{\mathrm{argmax}}~b_i(t)^T\tilde{\mu}(t)$. The posterior distribution $\mathcal{N}(\hat{\mu}(t), v^2B(t)^{-1})$ with $\hat{\mu}(t)$ defined in (\ref{muhat}) is easily derived by assuming a gaussian prior $\mathcal{N}(0_d,v^2I_d)$ on $\mu$ for some $v>0$ and that $r_i(t)$ given $\mu$ follows a gaussian distribution $\mathcal{N}(b_i(t)^T\mu, v^2).$ 


\citet{Agrawal} derived the high-probability upper bound of $R(T)$ for the TS algorithm. This bound does not require the Bayesian framework nor the gaussian assumption for the rewards. Under (\ref{linear}) and $R$-sub-gaussianity of the errors, it can be shown that with probability greater than $1-\delta,$ 
\begin{align}R(T)&\leq O\big(d^{\frac{3}{2}}\sqrt{T\mathrm{log}(Td)\mathrm{log}(T/\delta)}(\sqrt{\mathrm{log}(1+T/d)}+\sqrt{\mathrm{log}(1/\delta)}~)\big).\label{tsregret}
\end{align}
The bound (\ref{tsregret}) matches the bound (\ref{ucbregret2}) by a factor of $\sqrt{d}\sqrt{\mathrm{log}(T)}$, which is the price for randomness. On the other hand, the TS algorithm does not require the {\bf for} loop in the UCB algorithm to compute the $s_{t,i}'s$ for each arm $i$.


Unlike aforementioned linear contextual MABs, adversarial contextual MABs do not impose any assumption on the functional form of $\theta_t(\cdot)$. Hence, the distribution of $r_i(t)$ is allowed to change over time, and it can also change adaptively depending on the history. In fact, we assume that an unknown adversary controls the value of $r_i(t)$ in a way that hampers the learner. In this more relaxed setting though, it is hard to achieve low $regret(t)$ with respect to the best choice $r_{a^*(t)}(t)$. Instead, the learner competes with a predefined, finite set of $K$ policies and the regret is defined with respect to the best policy in that set.   


The EXP4.P algorithm proposed by \citet{Beygelzimer et al.2} achieves $O(\sqrt{TN\mathrm{log}({K}/{\delta})})$ regret upper bound. However, for the new notion of regret to be close to the more conservative definition, $K$ should be as large as possible so as to contain the optimal policy which chooses $a^*(t)$ for every $t$, resulting in larger regret. {Therefore, when a simple parametric or semiparametric assumption is not considered so farfetched, algorithms that exploit this structure can lead to higher rewards.}   

\section{Semiparametric contextual MAB}\label{s3}
\citet{Greenewald} and \citet{Krishnamurthy} considered a middle ground between simple linear contextual MABs and complex adversarial MABs: a semiparametric contextual MAB. {In this section, we formally present the semiparametric contextual MAB problem and related works.}

\subsection{{Semiparametric additive reward model}}
 Hereinafter, we define $\mathcal{H}_{t-1}$ as the history until time $t-1$, i.e., $\mathcal{H}_{t-1}=\{a(\tau),r_{a(\tau)}(\tau), \{b_i(\tau)\}_{i=1}^N, \tau=1,\cdots, t-1\},$  and the filtration $\mathcal{F}_{t-1}$ as the union of $\mathcal{H}_{t-1}$ and the contexts at time $t$, i.e., $\mathcal{F}_{t-1}=\{\mathcal{H}_{t-1}, \{b_i(t)\}_{i=1}^N\}$ for $t=1,\cdots, T.$ 
Given $\mathcal{F}_{t-1}$, we assume that the expectation of the reward $r_i(t)$ can be decomposed into a time-invariant, linear component depending on action $(b_i(t)^T\mu)$ and a nonparametric component depending on time and possibly on $\mathcal{F}_{t-1}$, but not on the action ($\nu(t)$):
\begin{align}\mathbb{E}(r_i(t)|\mathcal{F}_{t-1})=\nu(t)+b_i(t)^T\mu.\label{semipara}\end{align}
In (\ref{semipara}), we do not impose any distributional assumption on $\nu(t)$ except that it is bounded, $|\nu(t)|\leq 1$. If $\nu(t)=0$, the problem is just a linear contextual MAB problem, whereas if $\nu(t)$ depends on the action as well, the reward distribution is completely nonparametric and can be addressed by adversarial MAB algorithms.


In the news article recommendation example, $\nu(t)$ can represent the baseline tendency of the user visiting at time $t$ to click any article in the Featured tab, regardless of the contents of the article. This baseline tendency can change in an unexpected manner, because different users visit at each time and even for the same user, the clicking tendency can change according to the user's mood or schedule, both of which cannot be captured as contextual information.  It is reasonable to assume that given this baseline tendency, the probability that the user clicks an article is linear with respect to context information of the article and the user.

Under (\ref{semipara}), we note that the optimal action $a^*(t)$ at time $t$ does not depend on $\nu(t)$ but only on the value of $\mu$, and the regret does not depend on $\nu(t)$ either: $$regret(t)=b_{a^*(t)}(t)^T\mu-b_{a(t)}(t)^T\mu.$$ However, $\nu(t)$ confounds the estimation of $\mu$. The nature of the bandit problem renders the distinction of $\nu(t)$ from the linear part especially difficult because only one observation is allowed at each time $t$. 
Moreover, under the partially adversarial model (\ref{semipara}), deterministic algorithms such as UCB algorithms turn out to be useless. This is because for deterministic algorithms, $a(t)\in\mathcal{F}_{t-1}$. Hence, if an adversary sets $\nu(t)\in \mathcal{F}_{t-1}$ to be $\nu(t)=-b_{a(t)}(t)^T\mu$, the observed reward is $r_{a(t)}(t)=\eta_{a(t)}(t)$ for all $t=1,\cdots,T$, and the algorithm cannot learn $\mu$. Therefore, we should capitalize on the randomness of action choice.  
 
Besides (\ref{semipara}), we make the usual assumption that given $\mathcal{F}_{t-1}$, the error $\eta_i(t):=r_i(t)-\mathbb{E}(r_i(t)|\mathcal{F}_{t-1})$ is $R$-sub-Gaussian for some $R>0$, i.e., for every $\lambda\in \mathbb{R},$ 
\begin{align}\mathbb{E}[\mathrm{exp}(\lambda\eta_i(t))|\mathcal{F}_{t-1}]\leq \mathrm{exp}({\lambda^2R^2}/{2}).\label{error2}\end{align}
Note that this assumption is satisfied whenever $r_i(t)\in [\nu(t)+b_i(t)^T\mu-R,\nu(t)+b_i(t)^T\mu+R]$. 
Also without loss of generality, we assume \begin{align}||b_i(t)||_2\leq 1,~ ||\mu||_2\leq 1,~ |\nu(t)|\leq 1,\label{bounddness}\end{align}  
where $||\cdot||_p$ denotes the $L_p$-norm.

\subsection{Related Work}


\citet{Greenewald} proposed the action-centered TS algorithm for the new reward model (\ref{semipara}). In their settings, they assumed that the first action is the base action, of which the context vector is $b_1(t)=0_d$ for all $t$. Hence, the expected reward of the base action is $\nu(t)$, which can vary with time and also in a way that depends on the past. 
\citet{Greenewald} followed the basic framework of the randomized, TS algorithm but in two stages. In the first stage, the learner selects one action among the non-base actions in the same way as in TS algorithm using random $\tilde{\mu}(t)$. Let this action be $\bar{a}(t)$. In the second stage, the learner chooses once again between $\bar{a}(t)$ and the base action using the distribution of $\tilde{\mu}(t)$. This finally chosen action is set as $a(t)$ and only this action is actually taken. In the second stage,  the probability of $a(t)=\bar{a}(t)$ is computed using the Gaussian distribution of $\tilde{\mu}(t)$,
$\mathbb{P}(a(t)=\bar{a}(t)|\mathcal{F}_{t-1},\bar{a}(t))=1-\psi\Big(\frac{-b_{\bar{a}(t)}(t)^T\hat{\mu}(t)}{vs_{t,\bar{a}(t)}(t)}\Big),$
where $\psi(\cdot)$ is the CDF of the standard Gaussian distribution. 


Instead of choosing $a(t)=\bar{a}(t)$ with this exact probability however, \citet{Greenewald} constrained the probability of not choosing the base action to lie in a predefined set $[p_{min}, p_{max}]\subset [0,1]$. This is to prevent the policy from converging to a determinisitic policy which can be ineffective in the mobile health setting that the authors considered. Hence, the algorithm selects $a(t)=\bar{a}(t)$ with probability $p_t=\mathrm{max}\Big(p_{min},~\mathrm{min}\Big(1-\psi\Big(\frac{-b_{\bar{a}(t)}(t)^T\hat{\mu}(t)}{vs_{t,\bar{a}(t)}(t)}\Big),~p_{max}\Big)\Big).$ Under this probability constraint, the definition of the optimal policy and $regret(t)$ changes accordingly. Let $\bar{a}^*(t)=\underset{2\leq i\leq N}{\mathrm{argmax}}~b_i(t)^T\mu$. Thus, $\bar{a}^*(t)$ is the optimal action among the non-base actions. Then the optimal policy chooses the action $a^*(t)=\bar{a}^*(t)$ with probability $\pi^*(t):=p_{max}I(b_{\bar{a}^*(t)}(t)^T\mu > 0)+p_{min}I(b_{\bar{a}^*(t)}(t)^T\mu\leq 0)$ and $a^*(t)=1$ with probability $1-\pi^*(t).$

To consistently estimate $\mu$, \citet{Greenewald} defined a pseudo-reward, $\hat{r}_{\bar{a}(t)}(t)=\{I(a(t)=\bar{a}(t))-p_t\}r_{a(t)}(t).$
An important property of the pseudo-reward is that its conditional expectation does not depend on $\nu(t)$. 
\citet{Greenewald} used this pseudo-reward instead of the actual reward $r_{a(t)}(t)$ for estimating $\mu$. They showed that the high probability upper bound of $R(T)$ for the action-centered TS algorithm matches that of the original TS algorithm for linear reward models, but by a constant factor $M=1/\{p_{min}(1-p_{max})\}$. 
{This factor $M$ can be large when we do not want to restrict action selection probabilities, i.e., when we want to set either $p_{min}=0$ or $p_{max}=1$.} 


\citet{Krishnamurthy} proposed the BOSE (Bandit Orthogonalized Semiparametric Estimation) algorithm for the semiparametric reward model (\ref{semipara}). {This algorithm takes  an action elimination method adapted from \citet{Even-Dar}. At each time $t$, an action $i$ is eliminated if there exists another action $j$ such that $\big(b_j(t)-b_i(t)\big)^T\hat{\mu}(t)> \omega \sqrt{(b_i(t)-b_j(t))^TV_t^{-1}(b_i(t)-b_j(t))},$ where $\omega$ is a predefined constant, $
\hat{\mu}(t)$ is an estimate of $\mu$, and $V_t$ is a $d$-dimensional matrix. The algorithm then picks up one action randomly among the survivors according to a particular distribution.}

For estimating $\mu$, \citet{Krishnamurthy} used a centering trick on the context vectors $b_i(t)$'s to cancel out $\nu(t)$. They proposed the following estimator for $\mu$:
\begin{align}\hat{\mu}(t)=\big(\gamma I_d+\sum_{\tau=1}^{t-1}X_{\tau}X_{\tau}^T\big)^{-1}\sum_{\tau=1}^{t-1}X_{\tau}r_{a(\tau)}(\tau),\label{Kmu_hat}\end{align}
where $X_{\tau}=b_{a(\tau)}(\tau)-\mathbb{E}(b_{a(\tau)}(\tau)|\mathcal{F}_{\tau-1})$ and $\gamma>0$. Given $\mathcal{F}_{\tau-1}$, we see that $\mathbb{E}(X_{\tau}|\mathcal{F}_{\tau-1})=0_d.$ Hence, $\{\sum_{\tau=1}^tX_{\tau}\}_{t=1}^{\infty}$ is a vector martingale process adapted to filtration $\{\mathcal{F}_{t}\}_{t=1}^{\infty}$. \citet{Krishnamurthy} derived a ($1-\delta$)-probability upper bound for $b^T(\hat{\mu}(t)-\mu)$ using a new concentration inequality for self-normalized vector-valued martingales established by \citet{delaPena2} and \citet{delaPena}. 

{The BOSE algorithm does not require any constraint on the action choice probabilities but achieves a $O(d\sqrt{T}\mathrm{log}(T/\delta))$ regret bound. This bound matches the best known regret bound (\ref{ucbregret2}) for linear reward models. However, the action elimination step requires $O(N^2)$ computations at each round. Also, the distribution used to select the action should satisfy a specific condition to guarantee the $O(d\sqrt{T}\mathrm{log}(T/\delta))$ regret bound.  {The authors only show that there exists a distribution to satisfy this condition when $N>2$. The construction of such distribution is not a trivial matter since it requires to solve a convex program with $N$ quadratic conditions at every iteration.} Furthermore, the bound of $b^T(\hat{\mu}(t)-\mu)$ is valid under $\gamma\geq 4d\mathrm{log}(9T)+8\mathrm{log}(4T/\delta)$ when $N>2$, which can dominate the denominator term of $\hat{\mu}(t)$ when $t$ is small. For example, when $d=35$ and $T=1900000$ as in the news article recommendation example in Section \ref{s6}, $\gamma\geq 2476.8$ if we take $\delta=0.1$. When $\gamma$ is set to be a tuning parameter, the BOSE algorithm requires in total two tuning parameters, including $\omega$ used in the action elimination step.}

\section{Proposed method}\label{s4}

In this paper, we propose a new algorithm for the semiparametric reward model (\ref{semipara}) which improves on the results of \citet{Greenewald} while keeping the framework of the TS algorithm. { Our method requires only $O(N)$ computations at each round, while \citet{Krishnamurthy} requires $O(N^2)$.  {An action selection distribution for every $N$ is given and does not need to be solved as in \citet{Krishnamurthy}.}  The proposed algorithm uses a new estimator $\hat{\mu}(t)$ for $\mu$ which enjoys a tighter high-probability upper bound {than (\ref{Kmu_hat})} without having to deal with any potentially big constant, $\gamma$.} We prove that the high-probability  upper  bound  of  the  regret $R(T)$ incurred  by  the proposed algorithm  has  the  same  order as  the  TS  algorithm  for  linear  reward  models without the need to restrict action choice probabilities as in \citet{Greenewald}. 

\subsection{Proposed algorithm}

\begin{algorithm}
\caption{Proposed TS algorithm}\label{prop_alg}
\begin{algorithmic}
\State Set $B=I_d$, $y=0_d$, $v=(2R+6)\sqrt{{6}d\mathrm{log}(T/\delta)}$, $\delta\in (0,1).$
\For{$t=1,2,\cdots, T$}
\State Compute $\hat{\mu}(t)=B^{-1}y.$
\State  Sample $\tilde{\mu}(t)$ from distribution $\mathcal{N}(\hat{\mu}(t),v^2B^{-1})$.
\State  Pull arm $a(t)=\underset{1\leq i\leq N}{\mathrm{argmax}}~b_i(t)^T\tilde{\mu}(t)$ and get reward $r_{a(t)}(t).$
\For{$i=1,\cdots,N$}
\State  Compute $\pi_i(t)=\mathbb{P}(a(t)=i|\mathcal{F}_{t-1}).$
\EndFor 
\State  Update $B$ and $y$:
\State  $B\leftarrow B+\big(b_{a(t)}(t)-\bar{b}(t)\big)\big(b_{a(t)}(t)-\bar{b}(t)\big)^T+\sum_{i=1}^N\pi_i(t)\big(b_i(t)-\bar{b}(t)\big)\big(b_i(t)-\bar{b}(t)\big)^T,$ 
\State  $y\leftarrow y+2\big(b_{a(t)}(t)-\bar{b}(t)\big)r_{a(t)}(t).$
\EndFor
\end{algorithmic}
\end{algorithm}

Besides (\ref{semipara}), we make the same assumptions as in Section \ref{s3}, (\ref{error2}) and (\ref{bounddness}). The proposed Algorithm \ref{prop_alg} follows the framework of the TS algorithm with two major modifications: the mean and variance of $\tilde{\mu}(t)$. First, we propose a new estimator $\hat{\mu}(t)$ of $\mu$ for the mean of $\tilde{\mu}(t)$:
\begin{align}\hat{\mu}(t)&=\Big(I_d+\hat{\Sigma}_t+\Sigma_t\Big)^{-1}\sum_{\tau=1}^{t-1}2X_{\tau}r_{a(\tau)}(\tau),\label{newmu_hat}\end{align}
where $\hat{\Sigma}_t\!=\!\sum_{\tau\!=\!1}^{t\!-\!1}X_{\tau}X_{\tau}^T$ and $\Sigma_t\!=\!\sum_{\tau\!=\!1}^{t\!-\!1}\mathbb{E}(X_{\tau}X_{\tau}^T|\mathcal{F}_{\tau-1})$. Compared to (\ref{Kmu_hat}), we note that the proposed estimator stabilizes the denominator using a new term $\Sigma_t$ instead of $\gamma I_d$. {As a result, we do not need an extra tuning parameter.} Hereinafter, let $\bar{b}(\tau)$ denote $\mathbb{E}(b_{a(\tau)}(\tau)|\mathcal{F}_{\tau-1})$ for simplicity. This term can be calculated as $\bar{b}(\tau)=\mathbb{E}\big(\sum_{i=1}^NI(a(\tau)=i)b_i(\tau)\big|\mathcal{F}_{\tau-1}\big)=\sum_{i=1}^N\pi_i(\tau)b_i(\tau),$ where $\pi_i(\tau)=\mathbb{P}(a(\tau)=i|\mathcal{F}_{\tau-1})$ is the probability of pulling the $i$-th arm at time $\tau$, which is determined by the distribution of $\tilde{\mu}(\tau)$. Also, the covariance $\mathbb{E}(X_{\tau}X_{\tau}^T|\mathcal{F}_{\tau-1})$ can be computed as $\mathbb{E}(X_{\tau}X_{\tau}^T|\mathcal{F}_{\tau-1})
=\sum_{i=1}^N\pi_i(\tau)(b_{i}(\tau)-\bar{b}(\tau))(b_{i}(\tau)-\bar{b}(\tau))^T.$
As for the variance of $\tilde{\mu}(t)$, we propose $v^2B(t)^{-1}$, where $v=(2R+6)\sqrt{6d\mathrm{log}(T/\delta)}$ and $B(t)=I_d+\hat{\Sigma}_t+\Sigma_t.$

In the following theorem, we establish a high-probability regret upper bound for the proposed algorithm.
\begin{thm}\label{mainthm} Under (\ref{semipara}), (\ref{error2}), and (\ref{bounddness}), the regret of Algorithm \ref{prop_alg} is bounded as follows. For $\forall \delta\in (0,1)$, with probability $1-\delta$, \begin{align*}R(T)&\leq O\big(d^{3/2}\sqrt{T}\sqrt{\mathrm{log}(Td)\mathrm{log}(T/\delta)}\big(\sqrt{\mathrm{log}(1+T/d)}+\sqrt{\mathrm{log}(1/\delta)}\big)\big).\end{align*}
\end{thm}
\noindent This bound matches the bound (\ref{tsregret}) of the original TS algorithm for linear reward models.
{The proof of Theorem \ref{mainthm} essentially follows the lines of the proof given by \citet{Agrawal} with some modifications. A complete proof is presented in the Appendix. The main contribution of this paper is a new theorem for the first stage, which bounds $|(b_i(t)-\bar{b}(t))^T(\hat{\mu}(t)-\mu)|$ with high probability with respect to the new estimator (\ref{newmu_hat}).} 
\begin{thm}\label{newmuhatbound}Let the event $E^{\hat{\mu}}(t)$ be defined as follows:
$$E^{\hat{\mu}}(t)=\big\{\forall i: |\big(b_i(t)-\bar{b}(t)\big)^T(\hat{\mu}(t)-\mu)|\leq l(t)s_{t,i}^c\big\},$$
where $s_{t,i}^c=\sqrt{\big(b_i(t)-\bar{b}(t)\big)^TB(t)^{-1}\big(b_i(t)-\bar{b}(t)\big)}$ and $l(t)=(2R+6)\sqrt{d\mathrm{log}(6t^3/\delta)}+1$. Then for all $t\geq 1$, for any $0<\delta<1$, $\mathbb{P}(E^{\hat{\mu}}(t))\geq 1-\frac{\delta}{t^2}.$ 
\end{thm}
 
\subsection{A sketch of proof for Theorem \ref{newmuhatbound}}

 By decomposition of $(\hat{\mu}(t)-\mu)$, 
 \begin{align*}\hat{\mu}(t)-\mu
&=B(t)^{-1}\sum_{\tau=1}^{t-1}2X_{\tau}r_{a(\tau)}(\tau)-\mu\nonumber\\
&=B(t)^{-1}\big\{\sum_{\tau=1}^{t-1}2X_{\tau}\eta_{a(\tau)}(\tau)+\sum_{\tau=1}^{t-1}2X_{\tau}\big(\nu(\tau)+{\bar{b}(\tau)^T\mu}\big)-\mu+\sum_{\tau=1}^{t-1}D(\tau)\mu\big\},\end{align*}
where $D(\tau)=X_{\tau}X_{\tau}^T-\mathbb{E}\big(X_{\tau}X_{\tau}^T|\mathcal{F}_{\tau-1}\big)$.
Let $b_i^c(t):=b_i(t)-\bar{b}(t)$. Hereinafter, we define $||x||_A:=\sqrt{x^TAx}$ for any $d$-dimensional vector $x$ and any $d\times d$ matrix $A$. By Cauchy-Schwarz inequality,
\begin{align}
\big|b_i^c(t)^T(\hat{\mu}(t)-\mu)\big|
&\leq s_{t,i}^c\{2C_1+2C_2+C_3+C_4\} \label{wrapup}, 
\end{align}
where
\begin{align*}
C_1&=\big|\big|\sum_{\tau=1}^{t-1}X_{\tau}\eta_{a(\tau)}(\tau)\big|\big|_{B(t)^{-1}},\\
C_2&=\big|\big|\sum_{\tau=1}^{t-1}X_{\tau}\big(\nu(\tau)+\bar{b}(\tau)^T\mu\big)\big|\big|_{B(t)^{-1}},\\
C_3&=\big|\big|\sum_{\tau=1}^{t-1}D(\tau)\mu\big|\big|_{B(t)^{-1}},
\end{align*}
and $C_4=\big|\big|\mu\big|\big|_{B(t)^{-1}}$.
First, we have $C_4\leq 1.$ Now we need to bound $C_1,$ $C_2,$ and $C_3$. {First, the term $C_1$ is a familiar term, which we can bound using the technique of \citet{Abbasi-Yadkori}.}
Since $\eta_{a(\tau)}(\tau)$ is R-sub-gaussian given $\mathcal{F}_{\tau-1}$ and $a(\tau)$ while $X_{\tau}$ is fixed given $\mathcal{F}_{\tau-1}$ and $a(\tau)$, we have for any $\lambda \in \mathbb{R}^d,$
\begin{align}
\mathbb{E}\Big[\mathrm{exp}\Big(\frac{\eta_{a(\tau)}(\tau)}{R}\lambda^TX_{\tau}-\frac{1}{2}\lambda^TX_{\tau}X_{\tau}^T\lambda\Big)\Big|\mathcal{F}_{\tau-1}, a(\tau)\Big] &\leq 1.\nonumber
\end{align}
Then it follows,\begin{align} \mathbb{E}\Big[\mathrm{exp}\Big(\lambda^T\sum_{\tau=1}^{t-1}\frac{\eta_{a(\tau)}(\tau)}{R}X_{\tau}-\frac{1}{2}\lambda^T\hat{\Sigma}_t\lambda\Big)\Big] &\leq 1.\label{condfor(i)}
\end{align}
From (\ref{condfor(i)}), we can apply the following lemma, which is a simplified version of the Corollary 4.3 of \citet{delaPena}.
\begin{lem}\label{delaPena04}
Let $X_{\tau}\in \mathbb{R}^d$ and $c_{\tau}\in \mathbb{R}$ be some random variables, $\tau=1,\cdots,t$. Suppose $\exists d\times d$ positive semi-definite matrix $A(t)$ such that for any $\lambda\in \mathbb{R}^d,$
\begin{equation}\mathbb{E}\Big[\mathrm{exp}\Big\{\lambda^T\sum_{\tau=1}^tX_{\tau}c_{\tau}-\frac{1}{2}\lambda^TA(t)\lambda\Big\}\Big]\leq 1.\label{keycond} \end{equation}
Then for any $0<\delta<1$ and any positive definite matrix $Q$, with probability at least $1-\delta,$
$${\Big|\Big|\sum_{\tau=1}^tX_{\tau}c_{\tau}\Big|\Big|_{(Q+A(t))^{-1}}^2} \leq {\mathrm{log}\Big(\frac{det\big(Q+A(t)\big)/det\big(Q\big)}{\delta^2}\Big)}.$$
\end{lem}
Taking $c_{\tau}=\frac{1}{R}\eta_{a(\tau)}(\tau)$, $Q=I_d+\Sigma_t,$ and $A(t)=\hat{\Sigma}_t,$ we see that (\ref{condfor(i)}) corresponds to the condition (\ref{keycond}) of the lemma. Also, 
$C_1= R\Big|\Big|\sum_{\tau=1}^{t-1}X_{\tau}c_{\tau}\Big|\Big|_{(Q+A(t))^{-1}}. $
Therefore by Lemma \ref{delaPena04}, for any $0<\delta<1$, with probability at least $1-\frac{\delta}{3t^2},$
\begin{align}
C_1
&\leq R\sqrt{\mathrm{log}\Big(\frac{det(Q+A(t))/det(Q)}{(\delta/(3t^2))^2}\Big)} \nonumber \\
&\leq R\sqrt{\mathrm{log}\Big(\frac{det(Q+A(t))}{(\delta/(3t^2))^2}\Big)}= R\sqrt{\mathrm{log}\Big(\frac{det(B(t))}{(\delta/(3t^2))^2}\Big)}. \label{1stbound}
\end{align}
Now, we need to bound $C_2$ and $C_3$, which are terms that arise due to the $\nu(\tau)$'s and the use of a new estimator (\ref{newmu_hat}). Although $C_2$ looks similar to $C_1$, the term $(\nu(\tau)+\bar{b}(\tau)^T\mu)$ is not sub-Gaussian, so we can no longer use the technique of \citet{Abbasi-Yadkori}. Instead, we have $\mathbb{E}[X_{\tau}|\mathcal{F}_{\tau-1}]=0.$ To bound a similar term to $C_2$, \citet{Krishnamurthy} proposed to use Lemma 7 of \citet{delaPena2} for vector-valued martingales to derive an inequality analogous to (\ref{condfor(i)}). 
Using Lemma 7 of \citet{delaPena2}, we can prove that
for any $\lambda\in\mathbb{R}^d,$
\begin{equation}\mathbb{E}\Big[\mathrm{exp}\Big\{\lambda^T\sum_{\tau=1}^{t-1}X_{\tau}c_{\tau}-\frac{1}{2}\lambda^T\big(\hat{\Sigma}_t+\Sigma_t\big)\lambda\Big\}\Big]\leq 1.\label{condfor(ii)}\end{equation}
where $c_{\tau}=\Big(\frac{\nu(\tau)+\bar{b}(\tau)^T\mu}{2}\Big)$. 
Taking $A(t)=\hat{\Sigma}_t+\Sigma_t$ and $Q=I_d$, (\ref{condfor(ii)}) corresponds to condition (\ref{keycond}). Also, 
$C_2=2\Big|\Big|\sum_{\tau=1}^{t-1}X_{\tau}c_{\tau}\Big|\Big|_{(Q+A(t))^{-1}}.$
Hence by Lemma \ref{delaPena04}, for any $0<\delta<1$, with probability at least $1-\frac{\delta}{3t^2},$ 
\begin{align}
C_2
&\leq 2\sqrt{\mathrm{log}\big({det(B(t))}/{(\delta/3t^2)^2}\big)}. \label{2ndbound}
\end{align}
{The final step is to bound $C_3$. However, $C_3$ does not take the form $||\sum X_{\tau}c_{\tau}||_{B(t)^{-1}}$, so we require additional work.}
Let $Y_{\tau}=D(\tau)\mu.$ Then, note that $Y_{\tau}\in \mathbb{R}^d$ and $\mathbb{E}\big[Y_{\tau}|\mathcal{F}_{\tau-1}\big]=0.$ We propose the following lemma.
\begin{lem} \label{condfor(iii)}
\begin{equation*}
\mathbb{E}\Big[\mathrm{exp}\Big\{\lambda^T\sum_{\tau=1}^{t-1}\frac{1}{\sqrt{2}}Y_{\tau}-\frac{1}{2}\lambda^T\big(\hat{\Sigma}_t+\Sigma_t\big)\lambda\Big\}\Big]\leq 1.
\end{equation*}
\end{lem}
The proof of Lemma \ref{condfor(iii)} is presented in the Appendix.
Taking $A(t)=\hat{\Sigma}_t+\Sigma_t$ and $Q=I_d$, Lemma \ref{condfor(iii)} corresponds to condition (\ref{keycond}). Also,
$C_3=2\Big|\Big|\sum_{\tau=1}^{t-1}\frac{1}{\sqrt{2}}Y_{\tau}\Big|\Big|_{(Q+A(t))^{-1}}.$
By Lemma \ref{delaPena04}, for any $0<\delta<1,$ with probability at least $1-\frac{\delta}{3t^2},$
\begin{align}
C_3
&\leq 2\sqrt{\mathrm{log}\big({det(B(t))}/{(\delta/3t^2)^2}\big)}. \label{3rdbound}
\end{align}
Plugging the bounds (\ref{1stbound}), (\ref{2ndbound}) and (\ref{3rdbound}) into (\ref{wrapup}) completes the proof. For any $0<\delta<1$, for all $i=1,\cdots,N,$ with probability at least $1-\frac{\delta}{t^2}$, 
\begin{align*}
\big|b_i^c(t)^T(\hat{\mu}(t)-\mu)\big|
&\leq s_{t,i}^c\big\{ (2R+6)\sqrt{\mathrm{log}\Big(\frac{det(B(t))}{(\delta/(3t^2))^2}\Big)}+1\big\}\\
&\leq l(t)s_{t,i}^c,
\end{align*}
 where the second inequality is due to the determinant-trace inequality, $det(B(t))\leq \big({trace(B(t))}/{d}\big)^d \leq ({2t})^d.$ 

\section{Simulation study}\label{s5}


We conduct simulation studies to evaluate the proposed algorithm, the original TS algorithm \cite{Agrawal}, the action-centered TS (ACTS) algorithm \citep{Greenewald} and the BOSE algorithm \citep{Krishnamurthy}. {We set $N\!\!=\!\!2$ or $6$ and $d\!=\!10$. We set the first action to be the base action, i.e., $b_1(t)\!=\!0_d$ for all $t$, and form the other context vectors as $b_{i}(t)\!=\![I(i\!\!=\!\!2)z_i(t)^T,\cdots,I(i\!\!=\!\!N)z_i(t)^T]^T$,  
where $z_i(t)\in\mathbb{R}^{d'}$, $d'\!\!=\!\!d/(N\!-\!1)$, and $z_i(t)$ is generated uniformly at random from the $d'$-dimensional unit sphere.} We generate $\eta_i(t)\overset{i.i.d.}{\sim}\mathcal{N}(0,0.01^2)$ and the rewards from (\ref{semipara}),
where we set $\mu=[-0.55,0.666,-0.09,-0.232,0.244,$ $0.55,-0.666, 0.09,0.232,-0.244]^T$ and consider three cases for $\nu(t)$:
(i) $\nu(t)\!=\!0$, (ii) $\nu(t)\!=\!-b_{a^*(t)}(t)^T\mu,$ (iii) $\nu(t)\!=\!\mathrm{log}(t+1).$
We conduct 30 simulations in total for each case. {Note that all four algorithms have one tuning parameter each that controls the degree of exploration. For the TS algorithms, the tuning parameter is $v$ in the variance of $\tilde{\mu}(t)$, and for the BOSE algorithm,  $\omega$ in the action elimination step.} For each algorithm, we use the value of the parameter which incurs minimum median regret over 30 simulations. These values can be found by grid search. 

\begin{figure*}
\begin{center}
~\includegraphics[width=16.5cm]{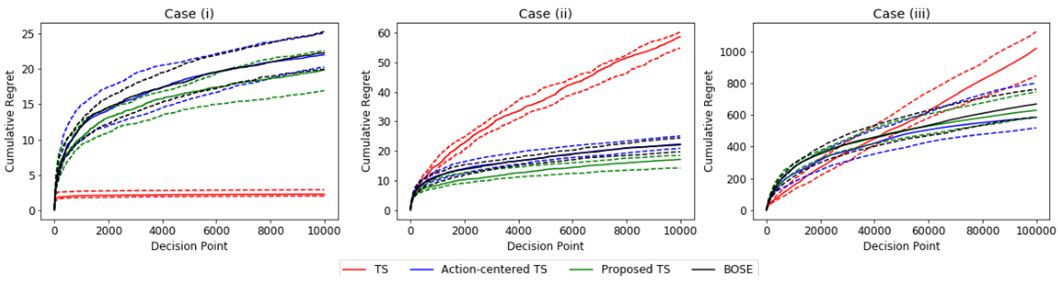}
\end{center}
\caption{Median (solid), 1st and 3rd quartiles (dashed) of cumulative regret over 30 simulations when $N=2$.}\label{figureN2}
\end{figure*}

\begin{figure*}
\begin{center}
~\includegraphics[width=16.5cm]{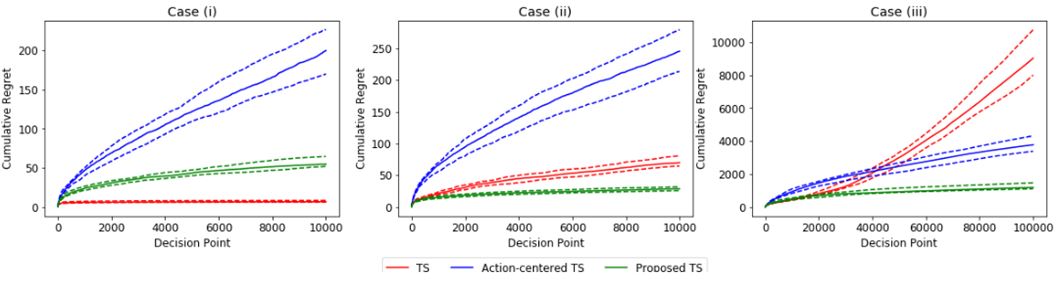}
\end{center}
\caption{Median (solid), 1st and 3rd quartiles (dashed) of cumulative regret over 30 simulations when $N=6$.}\label{figure}
\end{figure*}

\begin{table}[!h]\caption{Median of $R(T)$ over 30 simulations.}\label{table_simul}
\begin{center}
\begin{tabular}{p{1.3cm}p{0.5cm}p{0.5cm}p{0.6cm}p{0.6cm}p{0.6cm}p{0.6cm}}
\hline
\multirow{2}{*}{\bf Algorithms}& \multicolumn{3}{c}{$N=2$}& \multicolumn{3}{c}{$N=6$}\\
\cmidrule(lr){2-4}\cmidrule(lr){5-7}
& \multicolumn{1}{c}{(i)} & \multicolumn{1}{c}{(ii)} &\multicolumn{1}{c}{(iii)} &\multicolumn{1}{c}{(i)} &\multicolumn{1}{c}{(ii)} &\multicolumn{1}{c}{(iii)}   \\
\hline
\multicolumn{1}{l}{TS}& ~~2& 59& 1020&~~~~7 &~~70& 9026\\
\multicolumn{1}{l}{ACTS}& 22& 22& ~~586& 200& 246&3772\\
\multicolumn{1}{l}{Proposed TS}&20 &17 &~~629 &~~55 & ~~29&1186\\
\multicolumn{1}{l}{BOSE}&22 & 22&~~669 & \multicolumn{1}{c}{-----} & \multicolumn{1}{c}{-----}&\multicolumn{1}{c}{-----}\\
\hline
\end{tabular}
\end{center}
\end{table}

Figures \ref{figureN2} and \ref{figure} show the cumulative regret $R(t)$ according to time $t$. The solid lines represent the median values and the dashed lines represent the lower and upper 25\% percentiles. 
The values of $R(T)$ for each algorithm in each case are reported in Table \ref{table_simul}. 
Figure \ref{figureN2} summarizes the results when $N=2$. {When $\nu(t)=0$, the original TS algorithm achieves lowest cumulative regret. The proposed method shows the second best performance in this case, while the BOSE and ACTS algorithms are also competitive. In cases where $\nu(t)$ changes with time, the original TS algorithm hardly learns at all, while the three other methods  developed under the nonparametric intercept term are competitive. When $N=6$, Figure \ref{figure} exhibits a similar trend for the original TS and the proposed  algorithms as in the $N=2$ case. On the other hand, the BOSE algorithm has no explicit method so the results are not shown and the ACTS algorithm has much slower learning speed than the proposed TS method. }

\section{Real data analysis}\label{s6}

{We present the results of the proposed and existing methods using the R6A dataset provied by Yahoo! Webscope.} The dataset is observational log data of user clicks from May 1st, 2009 to May 10th, 2009, which corresponds to 45,811,883 user visits. At every visit, one article was chosen uniformly at random from 20 articles ($N=20$) and was displayed in the Featured tab of the Today module on Yahoo! front page. The reward $r_i(t)$ is binary, taking value $1$ if the visiting user clicked the $i$-th article, and $r_i(t)=0$ otherwise. For each article $i$, there is a context vector $b_i(t)\in\mathbb{R}^{35}$, which is constituted of 5 extracted user features, 5 extracted article features and their products. The extracted features were constructed from high-dimensional raw data for user and article features using a dimension reduction method of \citet{Chu09}.

Evaluating a new reinforcement learning policy retrospectively using observational log data is a challenging task itself and calls for off-policy evaluation methods. This is because in the log data, the rewards of the actions that were not chosen by the original logging policy are missing. In our log data, only the rewards of articles chosen by the uniform random policy are observed.

Denote the policy that we want to evaluate as $\texttt{A}$ and the total $T$-trial reward of $\texttt{A}$ as $G_{\texttt{A}}(T):=\mathbb{E}\big[\sum_{t=1}^Tr_{a(t)}(t)\big]$, where $a(t)$ is chosen by $\texttt{A}$. \citet{Li11} proposed an offline evaluation algorithm for estimating $G_{\texttt{A}}(T)$. Given the stream of events $(\boldsymbol{b}=\{b_i\}_{i=1}^N,a,r_{a})$ from log data, the algorithm picks up the events of which the chosen action $a$ matches the choice of $\texttt{A}$ and stacks them into the history of $\texttt{A}$. {Rewards in the history are used to construct the estimate $\hat{G}_{\texttt{A}}(T)$. In the case where $\texttt{A}$ is an online learning policy, the history is used to update the action selection distribution of $\texttt{A}$ as well.} Under the condition that $(\boldsymbol{b}(t), r(t))$ are i.i.d. and the logging policy is the uniform random policy, $\hat{G}_{\texttt{A}}(T)$ is shown to be unbiased. 
We note that the i.i.d. condition does not cover the case where $\nu(t)$ is adaptive to the past trials. Still, the conditional distribution of $\nu(t)$ given $\boldsymbol{b}(t)$ is not restricted.
 

We evaluate the uniform random policy, TS algorithm and the proposed algorithm using the method of \citet{Li11}. We use data of May 1st, 2009 as tuning data to choose the optimal exploration parameter $v$ for the TS algorithm and the proposed algorithm, respectively. Then we conduct main analysis on data from May 2nd to May 10th, 2009. Note that the method of \citet{Li11} picks up over ${1}/{N}={1}/{20}$ of the log data. This corresponds to over $T=1900000$. We fix the value of $T$ to $T=1900000$ a priori, and conduct the evaluation algorithm for 10 times on the same data for each policy. Since the evaluated policies are all randomized algorithms, each of the 10 runs pick up different actions, giving 10 different estimates. We report the mean, 1st quartile and 3rd quartile of the estimates for each policy in Table \ref{table}.  
\begin{table}\caption{Mean, 1st quartile (1st Q.) and 3rd quartile (3rd Q.) of user clicks achieved by each policy over 10 runs.}\label{table}
\begin{center}
\begin{tabular}{p{2.8cm}p{1.2cm}p{1.2cm}p{1.2cm}}
\hline
\bf Policies& \multicolumn{1}{c}{Mean} &\multicolumn{1}{c}{1st Q.} & \multicolumn{1}{c}{3rd Q.}\\
\hline
\multicolumn{1}{l}{Uniform policy}& 66696.7&66515.0& 66832.8\\
\multicolumn{1}{l}{TS algorithm}& 86907.0&85992.8& 88551.3\\
\multicolumn{1}{l}{Proposed TS}& 90689.7&90177.3 & 91166.3\\
\hline
\end{tabular}
\end{center}
\end{table}
We verify that the contextual bandit algorithms achieve substantially higher user click rates than the uniform random policy. Among the contextual bandit algorithms, the proposed algorithm 
increases the average user click rate by 4.4\% compared to the original TS algorithm. 

\section{Concluding remarks}\label{s7}

This paper proposes a new contextual MAB algorithm for a semiparametric reward model which is well suited to real problems where baseline rewards are bound to change with time.
 The proposed algorithm improves on existing methods that consider the same model. 
 Simulation study and real data analysis demonstrate the advantage of the proposed method.

\newpage

\section*{Appendix}

\setcounter{section}{0}
\renewcommand\thesection{\Alph{section}}

\section{Preliminaries}

\begin{lem}\label{boundons} \citep[Lemma 11 of ][]{Abbasi-Yadkori} Let $\{X_{t}\}_{t=1}^T$ be a sequence in $\mathbb{R}^d$ with $||X_{t}||_2\leq 1$, $Q$ a $d\times d$ positive definite matrix with $det(Q)\geq 1$ and $A(t)=\sum_{\tau=1}^{t-1}X_{\tau}X_{\tau}^T.$ Then, we have
$$\sum_{t=1}^TX_{t}^T\{Q+A(t)\}^{-1}X_{t}\leq 2\mathrm{log}\Big(\frac{det(Q+A(T+1))}{det(Q)}\Big).$$
\end{lem}


\begin{lem}\label{bercubound} \citep[Lemma 2.1 of][]{Bercu}
Let $x$ be a square integrable random variable with mean $0$ and variance $\sigma^2>0.$ Then,
$$\mathbb{E}\Big[\mathrm{exp}\Big(x-\frac{1}{2}x^2-\frac{1}{2}\sigma^2\Big)\Big]\leq 1. $$
\end{lem}

\begin{lem}\label{delaPena09} \citep[Lemma 7 of ][]{delaPena2} Let $X_{\tau}\in\mathbb{R}^d$ be $\mathcal{F}_{\tau}$-measurable for some filtration $\{\mathcal{F}_{\tau}\}_{\tau=1}^t,$ $\mathbb{E}\big[X_{\tau}|\mathcal{F}_{\tau-1}\big]=0,$ and $||X_{\tau}||_2\leq B$ for some constant $B$, $\tau=1,\cdots,t.$ Let $c_{\tau}\in \mathbb{R}$ be $\mathcal{F}_{\tau}$-measurable, $|c_{\tau}|\leq 1$ and $X_{\tau} \perp c_{\tau} |\mathcal{F}_{\tau-1}.$ Then for any $\lambda \in \mathbb{R}^d$, 
$$\mathbb{E}\Big[\mathrm{exp}\Big\{\lambda^T\sum_{\tau=1}^tX_{\tau}c_{\tau}-\frac{1}{2}\lambda^T\Big(\sum_{\tau=1}^tX_{\tau}X_{\tau}^T+\sum_{\tau=1}^t\mathbb{E}\big[X_{\tau}X_{\tau}^T|\mathcal{F}_{\tau-1}\big]\Big)\lambda\Big\}\Big]\leq 1.$$
\end{lem}
\begin{proof}
\noindent Taking $x=\lambda^TX_{\tau}c_{\tau}$, we have from Lemma \ref{bercubound}, 
$$\mathbb{E}\Big[\mathrm{exp}\Big\{\lambda^TX_{\tau}c_{\tau}-\frac{1}{2}\lambda^T\Big(c_{\tau}^2X_{\tau}X_{\tau}^T+\mathbb{E}\big[c_{\tau}^2X_{\tau}X_{\tau}^T|\mathcal{F}_{\tau-1}\big]\Big)\lambda\Big\}\Big|\mathcal{F}_{\tau-1}\Big]\leq 1.$$
Since $c_{\tau}^2\leq 1$ and $X_{\tau}X_{\tau}^T$ is positive semi-definite, 
$$\mathbb{E}\Big[\mathrm{exp}\Big\{\lambda^TX_{\tau}c_{\tau}-\frac{1}{2}\lambda^T\Big(X_{\tau}X_{\tau}^T+\mathbb{E}\big[X_{\tau}X_{\tau}^T|\mathcal{F}_{\tau-1}\big]\Big)\lambda\Big\}\Big|\mathcal{F}_{\tau-1}\Big]\leq 1.$$
\end{proof}

\begin{lem}\label{gaussianbound} \citep{Abramowitz} If $Z\sim \mathcal{N}(m,\sigma^2),$ for any $z\geq 1$, $$\frac{1}{2\sqrt{\pi}z}\mathrm{exp}\Big(-\frac{z^2}{2}\Big)\leq \mathbb{P}\big(|Z-m|>z\sigma\big)\leq \frac{1}{\sqrt{\pi}z}\mathrm{exp}\Big(-\frac{z^2}{2}\Big).$$
\end{lem}

\section{Proof of Theorem 4.2}

The proof of Theorem 4.2 follows the proof sketch of Section 4.2. 

\subsection{Proof of (14)}

Take $c_{\tau}=\Big(\frac{\nu(\tau)+\bar{b}(\tau)^T\mu}{2}\Big).$ Since $\mathbb{E}\big[X_{\tau}|\mathcal{F}_{\tau-1}\big]=0,$ $|c_{\tau}|\leq 1,$ and $X_{\tau}\perp c_{\tau}|\mathcal{F}_{\tau-1},$ we can apply Lemma \ref{delaPena09}, i.e., for any $\lambda\in \mathbb{R}^d, $
\begin{equation}\mathbb{E}\Big[\mathrm{exp}\Big\{\lambda^T\sum_{\tau=1}^{t-1}X_{\tau}c_{\tau}-\frac{1}{2}\lambda^T\Big(\sum_{\tau=1}^{t-1}X_{\tau}X_{\tau}^T+\sum_{\tau=1}^{t-1}\mathbb{E}[X_{\tau}X_{\tau}^T|\mathcal{F}_{\tau-1}]\Big)\lambda\Big\}\Big]\leq 1.\nonumber\end{equation}

\subsection{Proof of Lemma 4.4}

By Lemma \ref{delaPena09}, for any $\lambda\in\mathbb{R}^d$, 
\begin{equation*}\mathbb{E}\Big[\mathrm{exp}\Big\{\lambda^T\sum_{\tau=1}^{t-1}\frac{1}{\sqrt{2}}Y_{\tau}-\frac{1}{2}\lambda^T\Big(\frac{1}{2}\sum_{\tau=1}^{t-1}Y_{\tau}Y_{\tau}^T+\frac{1}{2}\sum_{\tau=1}^{t-1}\mathbb{E}\big[Y_{\tau}Y_{\tau}^T|\mathcal{F}_{\tau-1}\big]\Big)\lambda\Big\}\Big]\leq 1.\end{equation*}
Here, \begin{align}\lambda^TY_{\tau}Y_{\tau}^T\lambda&=\lambda^TD({\tau})\mu\mu^TD({\tau})\lambda\nonumber\\
&=\big\{\big(D(\tau)\lambda\big)^T\mu\big\}^2\nonumber\\
&\leq \mu^T\mu\big(D(\tau)\lambda\big)^T\big(D(\tau)\lambda\big)~~(\because \text{Cauchy-Schwarz inequality})\nonumber\\
&\leq \big(D(\tau)\lambda\big)^T\big(D(\tau)\lambda\big) = \lambda^TD(\tau)^2\lambda, \label{(i)}
\end{align}
and 
\begin{align}
\lambda^T\mathbb{E}\big[Y_{\tau}Y_{\tau}^T|\mathcal{F}_{\tau-1}\big]\lambda&\leq \lambda^T\mathbb{E}\big[D(\tau)^2|\mathcal{F}_{\tau-1}\big]\lambda. \label{(ii)}
\end{align}
Let $L=X_{\tau}X_{\tau}^T$ and $K=\mathbb{E}\big[X_{\tau}X_{\tau}^T|\mathcal{F}_{\tau-1}\big].$ Then,
\begin{align}
\lambda^TD(\tau)^2\lambda&=\lambda^T\big(L-K\big)^2\lambda \nonumber\\
&=\lambda^TL^2\lambda+\lambda^TK^2\lambda+2\lambda^TL(-K)\lambda\nonumber\\
&\leq \lambda^TL^2\lambda+\lambda^TK^2\lambda+2\sqrt{\lambda^TL^2\lambda~ \lambda^TK^2\lambda}~~(\because \text{Cauchy-Schwarz inequality})\nonumber\\
&\leq 2\lambda^TL^2\lambda+2\lambda^TK^2\lambda. \label{(iii)}
\end{align}
Also, \begin{align}
\mathbb{E}\big[D(\tau)^2|\mathcal{F}_{\tau-1}\big]&=\mathbb{E}\big[(L-K)^2|\mathcal{F}_{\tau-1}\big] \nonumber \\
&=\mathbb{E}\big[L^2|\mathcal{F}_{\tau-1}\big]-\mathbb{E}\big[L|\mathcal{F}_{\tau-1}\big]K-K\mathbb{E}\big[L|\mathcal{F}_{\tau-1}\big]+K^2\nonumber \\
&=\mathbb{E}\big[L^2|\mathcal{F}_{\tau-1}\big]-K^2 ~~~(\because \mathbb{E}\big[L|\mathcal{F}_{\tau-1}\big]=K) \nonumber \\
\Rightarrow \lambda^T\mathbb{E}\big[D(\tau)^2|\mathcal{F}_{\tau-1}\big]\lambda &\leq 2\lambda^T\mathbb{E}\big[D(\tau)^2|\mathcal{F}_{\tau-1}\big]\lambda \nonumber\\
&=2\lambda^T\mathbb{E}\big[L^2|\mathcal{F}_{\tau-1}\big]\lambda-2\lambda^TK^2\lambda. \label{(iv)}
\end{align}
Due to (\ref{(i)}), (\ref{(ii)}), (\ref{(iii)}) and (\ref{(iv)}), 
\begin{align*}\lambda^T\Big(Y_{\tau}Y_{\tau}^T+\mathbb{E}\big[Y_{\tau}Y_{\tau}^T|\mathcal{F}_{\tau-1}\big]\Big)\lambda&\leq 2\lambda^T\Big(L^2 + \mathbb{E}\big[L^2|\mathcal{F}_{\tau-1}\big]\Big)\lambda\\
&\leq 2\lambda^T\Big(X_{\tau}X_{\tau}^T+\mathbb{E}\big[X_{\tau}X_{\tau}^T|\mathcal{F}_{\tau-1}\big]\Big)\lambda,  \end{align*}
where the last inequality is due to $L=X_{\tau}X_{\tau}^T$ and $X_{\tau}^TX_{\tau}\leq 1$.
Therefore, for any $\lambda\in\mathbb{R}^d$, 

\begin{equation*}
\mathbb{E}\Big[\mathrm{exp}\Big\{\lambda^T\sum_{\tau=1}^{t-1}\frac{1}{\sqrt{2}}Y_{\tau}-\frac{1}{2}\lambda^T\Big(\sum_{\tau=1}^{t-1}X_{\tau}X_{\tau}^T+\sum_{\tau=1}^{t-1}\mathbb{E}\big[X_{\tau}X_{\tau}^T|\mathcal{F}_{\tau-1}\big]\Big)\lambda\Big\}\Big]~~~~~~~~~~~~~~~~~~~~~~~~~~~~~~~~~~~~~~~~~~~~~
\end{equation*}
\begin{align*}&\leq\mathbb{E}\Big[\mathrm{exp}\Big\{\lambda^T\sum_{\tau=1}^{t-1}\frac{1}{\sqrt{2}}Y_{\tau}-\frac{1}{2}\lambda^T\Big(\frac{1}{2}\sum_{\tau=1}^{t-1}Y_{\tau}Y_{\tau}^T+\frac{1}{2}\sum_{\tau=1}^{t-1}\mathbb{E}\big[Y_{\tau}Y_{\tau}^T|\mathcal{F}_{\tau-1}\big]\Big)\lambda\Big\}\Big]\nonumber\\
&\leq 1.\end{align*}

\section{Proof of Theorem 4.1}


The proof of Theorem 4.1 follows the lines of \citet{Agrawal} with some modifications. We present the whole proof.

\begin{enumerate}[(a)]
\item The first stage is the derivation of a high-probability upper bound of $|(b_i(t)-\bar{b}(t))^T(\hat{\mu}(t)-\mu)|$. This is done in Theorem 4.2, which we restate here for concreteness.
\begin{thm}\label{newmuhatbound2}Let the event $E^{\hat{\mu}}(t)$ be defined as follows:
$$E^{\hat{\mu}}(t)=\big\{\forall i: |\big(b_i(t)-\bar{b}(t)\big)^T(\hat{\mu}(t)-\mu)|\leq l(t)s_{t,i}^c\big\},$$
where $s_{t,i}^c=\sqrt{\big(b_i(t)-\bar{b}(t)\big)^TB(t)^{-1}\big(b_i(t)-\bar{b}(t)\big)}$ and $l(t)=(2R+6)\sqrt{d\mathrm{log}(6t^3/\delta)}+1$. Then for all $t\geq 1$, for any $0<\delta<1$, $\mathbb{P}(E^{\hat{\mu}}(t))\geq 1-\frac{\delta}{t^2}.$ 
\end{thm}

\item {We next establish a high-probability upper bound for $|(b_i(t)-\bar{b}(t))^T(\tilde{\mu}(t)-\hat{\mu}(t))|$ in the following Proposition \ref{mutildebound}. The proof is a simple extension of \citet{Agrawal}, which uses Lemma \ref{gaussianbound} for gaussian random variables. }

\begin{prop}\label{mutildebound} Let the event $E^{\tilde{\mu}}(t)$ be defined as follows:
$$E^{\tilde{\mu}}(t)=\big\{\forall i: |\big(b_i(t)-\bar{b}(t)\big)^T(\tilde{\mu}(t)-\hat{\mu}(t))|\leq m(T)s_{t,i}^c\big\},$$
where $m(T)=v\sqrt{4d\mathrm{log}(Td)}.$ Then for all $t\geq 0$, $\mathbb{P}(E^{\tilde{\mu}}(t)|\mathcal{F}_{t-1})\geq 1-\frac{1}{T^2}$.
\end{prop}

\begin{proof}
Note that given $\mathcal{F}_{t-1}$, the values of $\big(b_i(t)-\bar{b}(t)\big)$, $B(t)$ and $\hat{\mu}(t)$ are fixed. Then,
\begin{align*}
|b_i^c(t)^T(\tilde{\mu}(t)-\hat{\mu}(t))|&=|b_i^c(t)^TvB(t)^{-1/2}\frac{1}{v}B(t)^{1/2}(\tilde{\mu}(t)-\hat{\mu}(t))|\\
&\leq v\sqrt{b_i^c(t)^TB(t)^{-1}b_i^c(t)} \Big|\Big|\frac{1}{v}B(t)^{1/2}(\tilde{\mu}(t)-\hat{\mu}(t))\Big|\Big|_2 \\
&=v s_{t,i}^c \Big|\Big|\frac{1}{v}B(t)^{1/2}(\tilde{\mu}(t)-\hat{\mu}(t))\Big|\Big|_2 \\
&= v s_{t,i}^c \sqrt{\sum_{j=1}^d||Z_j(t)||_2^2},
\end{align*}
where $Z_j(t)|\mathcal{F}_{t-1} \overset{i.i.d.}{\sim} \mathcal{N}(0,1)$ and the first inequality is due to Cauchy-Schwarz inequality.
Due to Lemma \ref{gaussianbound}, for fixed $j$ and $z\geq 1$,
$$\mathbb{P}\big(|Z_j(t)|>z~|~\mathcal{F}_{t-1}\big) \leq \frac{1}{\sqrt{\pi}z}\mathrm{exp\Big(-\frac{z^2}{2}\Big)}\leq \mathrm{exp\Big(-\frac{z^2}{2}\Big)}.$$ Setting $\mathrm{exp\big(-{z^2}/{2}\big)}=\frac{1}{dT^2}$, we have $z=\sqrt{2\mathrm{log}(dT^2)}\leq \sqrt{2\mathrm{log}(d^2T^2)}$ $=\sqrt{4\mathrm{log}(dT)}$. Hence,
\begin{align*}
\mathbb{P}\big(|Z_j(t)|>\sqrt{4\mathrm{log}(dT)}~|~\mathcal{F}_{t-1}\big)\leq \frac{1}{dT^2}\\
\Rightarrow \mathbb{P}\big(\forall j: |Z_j(t)|>\sqrt{4\mathrm{log}(dT)}~|~\mathcal{F}_{t-1}\big) \leq \frac{1}{T^2}.
\end{align*}
Thus, with probability at least $1-\frac{1}{T^2}$, for all $i=1,\cdots,N$,
\begin{align*}
|\big(b_i(t)-\bar{b}(t)\big)^T(\tilde{\mu}(t)-\hat{\mu}(t))|&\leq vs_{t,i}^c\sqrt{4d\mathrm{log}(dT)}= m(T)s_{t,i}^c.
\end{align*}
\end{proof}

\item Before proceeding, we divide the arms at each time into two groups: saturated and unsaturated arms. Let $g(T)=m(T)+l(T).$ An arm $i$ is saturated at time $t$ if $$\big(b_i(t)-\bar{b}(t)\big)^T\mu+g(T)s_{t,i}^c<\big(b_{a^*(t)}(t)-\bar{b}(t)\big)^T\mu,$$ and unsaturated otherwise. Note that the optimal arm $a^*(t)$ is unsaturated. Note also that from Stage (a) and Stage (b), $(b_i(t)-\bar{b}(t))^T\mu+g(T)s_{t,i}^c$ is an upper bound of $(b_i(t)-\bar{b}(t))^T\tilde{\mu}(t)$. Hence by definition, the saturated arms are the arms that have quite accurate values of $(b_i(t)-\bar{b}(t))^T\tilde{\mu}(t)$ so that their upper bound is lower than $(b_{a^*(t)}(t)-\bar{b}(t))^T\mu,$ enabling the algorithm to distinguish between them and the optimal arm.

\item {Next, we show in Proposition \ref{saturatebound} that the probability of playing saturated arms is bounded by a function of the probability of playing unsaturated arms. The proof is a simple extension of \citet{Agrawal}.}
\begin{prop}\label{saturatebound}  Let $C(t)$ be the set of saturated arms at time $t$, i.e., $C(t)=\{i: \big(b_i(t)-\bar{b}(t)\big)^T\mu+g(T)s_{t,i}^c<\big(b_{a^*(t)}(t)-\bar{b}(t)\big)^T\mu\}$. Given any filtration $\mathcal{F}_{t-1}$ such that $E^{\hat{\mu}}(t)$ is true, 
$$\mathbb{P}\big(a(t)\in C(t)|\mathcal{F}_{t-1}\big)\leq \frac{1}{p}\mathbb{P}\big(a(t)\notin C(t)|\mathcal{F}_{t-1}\big)+\frac{1}{pT^2},$$
where $p=\frac{1}{4\mathrm{e}\sqrt{2}\sqrt{\pi}}.$
\end{prop}
\begin{proof}
Since the algorithm pulls the arm $\underset{i}{\mathrm{argmax}}\{b_i(t)^T\tilde{\mu}(t)\},$ if $b_{a^*(t)}(t)^T\tilde{\mu}(t)>b_{j}(t)^T\tilde{\mu}(t)$ for every $j\in C(t),$ then $a(t)\notin C(t).$ Hence, 
\begin{align}\mathbb{P}\big(a(t)\notin C(t)|\mathcal{F}_{t-1}\big)&\geq \mathbb{P}\big(b_{a^*(t)}(t)^T\tilde{\mu}(t)>b_{j}(t)^T\tilde{\mu}(t),~ \forall j\in C(t)|\mathcal{F}_{t-1}\big)\nonumber\\
&=\mathbb{P}\big(b_{a^*(t)}^c(t)^T\tilde{\mu}(t)>b_{j}^c(t)^T\tilde{\mu}(t),~ \forall j\in C(t)|\mathcal{F}_{t-1}\big). \label{eq1}\end{align}
If $E^{\tilde{\mu}}(t)$ is additionally true, for $\forall j \in C(t),$
\begin{align*}
b_{j}^c(t)^T\tilde{\mu}(t)&\leq b_{j}^c(t)^T\mu+g(T)s_{t,j}^c~~~(\because E^{\hat{\mu}}(t)~\&~E^{\tilde{\mu}}(t))\\
&\leq b_{a^*(t)}^c(t)^T\mu.~~~(\because \text{definition of } C(t))
\end{align*}
Therefore,
\begin{align*}
\mathbb{P}\big(b_{a^*(t)}^c(t)^T\tilde{\mu}(t)>b_{j}^c(t)^T\tilde{\mu}(t),~ \forall j\in C(t)|\mathcal{F}_{t-1}\big)+\Big(1- \mathbb{P}\big(E^{\tilde{\mu}}(t)|\mathcal{F}_{t-1}\big)\Big)~~~~~~~~~~~~~~~~~
\end{align*}
\begin{align}
&\geq \mathbb{P}\big(b_{a^*(t)}^c(t)^T\tilde{\mu}(t)>b_{a^*(t)}^c(t)^T\mu|\mathcal{F}_{t-1}\big).\label{eq2}
\end{align}

\noindent Given $E^{\hat{\mu}}(t),$ $|b_{a^*(t)}^c(t)^T(\hat{\mu}(t)-\mu)|\leq l(T)s_{t,a^*(t)}^c.$ Thus by Lemma \ref{gaussianbound}, 
\begin{align}
(\ref{eq2})&=\mathbb{P}\Big(~\frac{b_{a^*(t)}^c(t)^T(\tilde{\mu}(t)-\hat{\mu}(t))}{v s_{t,a^*(t)}^c}>\frac{b_{a^*(t)}^c(t)^T(\mu-\hat{\mu}(t))}{v s_{t,a^*(t)}^c}\Big|\mathcal{F}_{t-1}\Big)\nonumber\\
&\geq \mathbb{P}\Big(~Z(t)>\frac{l(T)}{v}\Big|\mathcal{F}_{t-1}\Big)\nonumber\\
&\geq \frac{1}{4\sqrt{\pi}z}\mathrm{exp}\Big(-\frac{z^2}{2}\Big)\geq {p}, \label{eq3}
\end{align}
where $Z(t)|\mathcal{F}_{t-1}\sim \mathcal{N}(0,1)$ and $z={l(T)}/v.$ Therefore, due to (\ref{eq1}), (\ref{eq2}), (\ref{eq3}) and Proposition \ref{mutildebound}, $$\mathbb{P}\big(a(t)\notin C(t)|\mathcal{F}_{t-1}\big)\geq p-\frac{1}{T^2}.$$
$$\Rightarrow \frac{\mathbb{P}\big(a(t)\in C(t)|\mathcal{F}_{t-1}\big)}{\mathbb{P}\big(a(t)\notin C(t)|\mathcal{F}_{t-1}\big)+\frac{1}{T^2}}\leq \frac{1}{p}.$$
\end{proof}

\item { Next in Proposition \ref{stage(e)rbound}, we use Proposition \ref{saturatebound} and the definition of unsaturated arms to show that the regret can be bounded by a factor of $s_{t,a(t)}^c$ in expectation. }
\begin{prop}\label{stage(e)rbound}
Given any filtration $\mathcal{F}_{t-1}$ such that $E^{\hat{\mu}}(t)$ is true,
$$\mathbb{E}\big[regret(t)|\mathcal{F}_{t-1}\big]\leq \frac{5g(T)}{p}\mathbb{E}\big[s_{t,a(t)}^c|\mathcal{F}_{t-1}\big]+\frac{3g(T)}{pT^2}.$$
\end{prop}
\begin{proof}
Let $\tilde{a}(t)=\underset{i\notin C(t)}{\mathrm{argmin}}~s_{t,i}^c$. This value is determined by $\mathcal{F}_{t-1}$. Under both $E^{\hat{\mu}}(t)$ and $E^{\tilde{\mu}}(t)$,
\begin{align*}
b_{a^*(t)}^c(t)^T\mu&= b_{a^*(t)}^c(t)^T\mu-b_{\tilde{a}(t)}^c(t)^T\mu+b_{\tilde{a}(t)}^c(t)^T\mu\\
&\leq g(T)s_{t,\tilde{a}(t)}^c+b_{\tilde{a}(t)}^c(t)^T\mu\\
&\leq g(T)s_{t,\tilde{a}(t)}^c+b_{\tilde{a}(t)}^c(t)^T\tilde{\mu}(t)+g(T)s_{t,\tilde{a}(t)}^c\\
&\leq 2g(T)s_{t,\tilde{a}(t)}^c+b_{a(t)}^c(t)^T\tilde{\mu}(t)\\
&\leq 2g(T)s_{t,\tilde{a}(t)}^c+b_{a(t)}^c(t)^T{\mu}+g(T)s_{t,a(t)}^c\\
\Rightarrow regret(t)&\leq 2g(T)s_{t,\tilde{a}(t)}^c +g(T)s_{t,a(t)}^c,
\end{align*}
where the first inequality follows from the definition of unsaturated arms, the second and fourth inequalities from $E^{\hat{\mu}}(t)$ and $E^{\tilde{\mu}}(t)$, and the third inequality from the action selection mechanism.
Therefore, given $\mathcal{F}_{t-1}$ such that $E^{\hat{\mu}}(t)$ holds,
\begin{align}
\mathbb{E}\big[regret(t)|\mathcal{F}_{t-1}\big]
&\leq 2g(T)s_{t,\tilde{a}(t)}^c+g(T)\mathbb{E}\big[s_{t,a(t)}^c|\mathcal{F}_{t-1}\big]+1-\mathbb{P}(E^{\tilde{\mu}}(t)|\mathcal{F}_{t-1})\nonumber\\
&\leq 2g(T)s_{t,\tilde{a}(t)}^c+g(T)\mathbb{E}\big[s_{t,a(t)}^c|\mathcal{F}_{t-1}\big]+\frac{1}{T^2}.\label{divid1}
\end{align}
Here,
\begin{align*}
s_{t,\tilde{a}(t)}^c&=s_{t,\tilde{a}(t)}^c\big\{\mathbb{P}(a(t)\in C(t)|\mathcal{F}_{t-1})+\mathbb{P}(a(t)\notin C(t)|\mathcal{F}_{t-1})\big\}\\
&\leq s_{t,\tilde{a}(t)}^c\Big\{\frac{2}{p}\mathbb{P}(a(t)\notin C(t)|\mathcal{F}_{t-1})+\frac{1}{pT^2}\Big\}\\
&=\frac{2}{p}\mathbb{E}\big(s_{t,\tilde{a}(t)}^cI\{a(t)\notin C(t)\}\big|\mathcal{F}_{t-1}\big)+\frac{s_{t,\tilde{a}(t)}^c}{pT^2}\\
&\leq \frac{2}{p}\mathbb{E}\big(s_{t,{a}(t)}^cI\{a(t)\notin C(t)\}\big|\mathcal{F}_{t-1}\big)+\frac{s_{t,\tilde{a}(t)}^c}{pT^2}\\
&\leq \frac{2}{p}\mathbb{E}\big(s_{t,{a}(t)}^c\big|\mathcal{F}_{t-1}\big)+\frac{1}{pT^2},
\end{align*}
where the first inequality is due to Proposition \ref{saturatebound} and the second inequality is due to the definition of $\tilde{a}(t)$. Combining this result with (\ref{divid1}), we have 
\begin{align*}
\mathbb{E}\big[regret(t)|\mathcal{F}_{t-1}\big]&\leq \frac{5g(T)}{p}\mathbb{E}\big(s_{t,{a}(t)}^c\big|\mathcal{F}_{t-1}\big)+\frac{3g(T)}{pT^2}.
\end{align*}
\end{proof}

\item Let $M_t=regret(t)I(E^{\hat{\mu}}(t))-\frac{5g(T)}{p}s_{t,a(t)}^c-\frac{3g(T)}{pT^2}.$ Then $|M_t|$ is bounded by $\frac{9g(T)}{p}.$ Also, due to Proposition \ref{stage(e)rbound}, $\{M_t\}_{t=1}^T$ is a bounded super-martingale difference process with respect to the filtration $\{\mathcal{F}_{t}\}_{t=1}^T$. Hence by Azuma-Hoeffding's inequality, for any $a\geq 0,$
$$\mathbb{P}\big(\sum_{t=1}^TM_t \geq a\big) \leq \mathrm{exp}\Big(-\frac{a^2}{2\sum_{t=1}^Tc_t^2}\Big),$$
where $c_t=\frac{9}{p}g(T).$ Setting $\mathrm{exp}\Big(-\frac{a^2}{2\sum_{t=1}^Tc_t^2}\Big)=\frac{\delta}{2}$, we have $a=\frac{9}{p}g(T)\sqrt{2T\mathrm{log}\big(\frac{2}{\delta}\big)}.$ Thus with probability at least $1-\frac{\delta}{2},$
\begin{align}
\sum_{t=1}^Tregret(t)I(E^{\hat{\mu}}(t))&\leq \frac{5g(T)}{p}\sum_{t=1}^Ts_{t,a(t)}^c+\frac{3g(T)}{pT}+\frac{9}{p}g(T)\sqrt{2Tlog\big(\frac{2}{\delta}\big)}.
\label{final1}\end{align}
 In Proposition \ref{statbound}, we show that $\sum_{t=1}^Ts_{t,a(t)}^c\leq \sqrt{2dT\mathrm{log}(1+T/d)}$ using Lemma \ref{boundons}.
\begin{prop}\label{statbound} $\sum_{t=1}^Ts_{t,a(t)}^c\leq \sqrt{2dT\mathrm{log}(1+T/d)}.$ 
\end{prop}
\begin{proof}
Take $X_{t}=b_{a(t)}(t)-\bar{b}(t)$, $Q=I_d$, and $A(t)=\sum_{\tau=1}^{t-1}X_{\tau}X_{\tau}^T$. Then by Jensen's inequality and Lemma \ref{boundons},
\begin{align*}
\sum_{t=1}^Ts_{t,a(t)}^c&\leq \sqrt{T\sum_{t=1}^T\{s_{t,a(t)}^c\}^2}~~(\because \text{ Jensen's inequality})\\
&= \sqrt{T\sum_{t=1}^TX_t^TB(t)^{-1}X_t}\\
&\leq \sqrt{T\sum_{t=1}^TX_t^T\{Q+A(t)\}^{-1}X_t} ~~(\because B(t)\succ Q+A(t))\\
&\leq \sqrt{2T\mathrm{\log}\Big(\frac{det(Q+A(T+1))}{det(Q)}\Big)} ~~(\because\text{ Lemma \ref{boundons}})\\
&\leq \sqrt{2dT\mathrm{\log}\Big(1+\frac{T}{d}\Big)}. ~~(\because \text{ determinant-trace inequality.})
\end{align*}
\end{proof}
\noindent Due to (\ref{final1}), Proposition \ref{statbound} and the definitions of $p$ and $g(T)$, we have with probability at least $1-\frac{\delta}{2}$, 
$$\sum_{t=1}^Tregret(t)I(E^{\hat{\mu}}(t))\leq O\Big(d^{3/2}\sqrt{T}\sqrt{\mathrm{log}(Td)\mathrm{log}(T/\delta)}\big(\sqrt{\mathrm{log}(1+T/d)}+\sqrt{\mathrm{log}(1/\delta)}\big)\Big).$$
Since $E^{\hat{\mu}}(t)$ holds for all $t$ with probability at least $1-\frac{\delta}{2}$ (Theorem \ref{newmuhatbound2}), $regret(t)I(E^{\hat{\mu}}(t))=regret(t)$ for all $t$ with probability at least $1-\frac{\delta}{2}$. Hence, with probability at least $1-\delta$, 
$$R(T)\leq  O\Big(d^{3/2}\sqrt{T}\sqrt{\mathrm{log}(Td)\mathrm{log}(T/\delta)}\big(\sqrt{\mathrm{log}(1+T/d)}+\sqrt{\mathrm{log}(1/\delta)}\big)\Big).$$
\end{enumerate}

\end{document}